\def\final{0}
\documentclass[11pt]{article}

\usepackage{algorithm, algorithmic}
\usepackage{amsmath, amssymb, amsthm}
\usepackage{framed}
\usepackage{fullpage, endnotes}
\usepackage{verbatim}
\usepackage{xcolor, tcolorbox}

\usepackage{caption}
\usepackage[capitalise]{cleveref}
\author{Mark Bun \thanks{Department of Computer Science, Boston University \textit{mbun@bu.edu}} \and Roi Livni \thanks{Department of Electrical Engineering, Tel Aviv University \textit{rlivni@tauex.tau.ac.il}}\and Shay Moran \thanks{Google AI Princeton. Part of this work was done while the author was at the Computer Science Department in Princeton University. \textit{shaymoran1@gmail.com}}}
\ifnum\final=0
\newcommand{\mynote}[1]{\marginpar{\tiny \sf #1}}
\else
\newcommand{\mynote}[1]{}
\fi
\newcommand{\mnote}[1]{}
\newcommand{\rnote}[1]{}
\newcommand{\snote}[1]{}
\newcommand{\correction}[2]{\color{black}#1\color{black}}
\newcommand{\rsuggestion}[1]{}

\newcommand{\cA}{\mathcal{A}}

\newcommand{\cD}{\mathcal{D}}

\newcommand{\cH}{\mathcal{H}}

\newcommand{\poly}{\mathrm{poly}}

\newcommand{\eps}{\varepsilon}
\newcommand{\argmax}{\operatorname{argmax}}

\DeclareMathOperator*{\Expectation}{\mathbb{E}}

\newtheorem{theorem}{Theorem}
\newtheorem{lemma}[theorem]{Lemma}

\newtheorem{corollary}[theorem]{Corollary}
\newtheorem{proposition}[theorem]{Proposition}
\newtheorem{observation}[theorem]{Observation}

\theoremstyle{definition}
\newtheorem{definition}[theorem]{Definition}

\newcommand{\loss}{\operatorname{loss}}
\newcommand{\freq}{\operatorname{freq}}

\newcommand{\Ldim}{\operatorname{Ldim}}
\newcommand{\soa}{\mathsf{SOA}}

\newcommand{\Hist}{\mathsf{Hist}}
\newcommand{\Gen}{\mathsf{GenericLearner}}

\DeclareCaptionFormat{myformat}{#3}
\title{An Equivalence Between Private Classification \\ and Online Prediction}

\begin{document}
\maketitle

\abstract{
We prove that every concept class with finite Littlestone dimension
can be learned by an (approximate) differentially-private algorithm.
This answers an open question of Alon et al.~(STOC 2019) who proved the converse statement 
(this question was also asked by  Neel et al.~(FOCS 2019)).
Together these two results yield an equivalence between online learnability and private PAC learnability.


We introduce a new notion of algorithmic stability called ``global stability'' which is essential to our proof and may be of independent interest.
We also discuss an application of our results to boosting the privacy and accuracy parameters of differentially-private learners.
}

\section{Introduction}

This paper continues the study of the close relationship between differentially-private learning and online learning.

\paragraph{Differentially-Private Learning.}
Statistical analyses and computer algorithms play significant roles in the decisions which shape modern society.
The collection and analysis of individuals' data drives computer programs which determine many critical outcomes, including the allocation of community resources, decisions to give loans, and school admissions.	

While data-driven and automated approaches have obvious benefits in terms of efficiency,
	they also raise the possibility of unintended negative impacts, especially against marginalized groups. This possibility
 	highlights the need for {\it responsible} algorithms that obey relevant ethical requirements
	(see e.g.~\cite{Oneil2016weapons}).

{\it Differential Privacy} (DP)~\cite{DworkMNS06} plays a key role in this context.
	Its initial (and primary) purpose was to provide a formal framework for ensuring individuals' privacy in the statistical analysis of large datasets.
	But it has also found use in addressing other ethical issues such as {\it algorithmic fairness} (see, e.g.~\cite{DworkHPRZ12, cummings19fairness}).

Many tasks which involve sensitive data arise in machine learning (e.g.\ in medical applications and in social networks).
	Consequently, a large body of practical and theoretical work has been dedicated to understand which learning tasks
	can be performed by DP learning algorithms. The simplest and most extensively studied model of learning is the private PAC model~\cite{Valiant84, KasiviswanathanLNRS11}, which captures binary classification tasks under differential privacy.
	A partial list of works on this topic includes~\cite{KasiviswanathanLNRS11,BeimelBKN14,BunNSV15,FeldmanX15,BeimelNS16,BunDRS18,Beimel19Pure,AlonLMM19,kaplan2019privately}.

In this manuscript we make progress towards characterizing what tasks are DP PAC-learnable 
	by demonstrating a qualitative equivalence with online-learnable tasks.
	
\paragraph{Online Learning.}

Online learning is a well-studied branch of machine learning which addresses algorithms making real-time predictions on sequentially arriving data.
	Such tasks arise in contexts including recommendation systems and advertisement placement.
	The literature on this subject is vast and includes several books, e.g.~\cite{cesabianchi06prediction,Shalev-Shwartz12book,Hazan16oco}.

{\it Online Prediction}, or {\it Prediction with Expert Advice} is a basic setting within online learning.
	Let~$\mathcal{H} = \{h:X\to \{\pm1\} \}$ be a class of predictors (also called experts) over a domain $X$.
	Consider an algorithm which observes examples $(x_1,y_1)\ldots (x_T,y_T)\in X\times\{\pm 1\}$ in a sequential manner.
	More specifically, in each time step $t$, the algorithm first observes the instance $x_t$, then predicts a label $\hat{y}_t\in\{\pm 1\}$, 
	and finally learns whether its prediction was correct.
	The goal is to minimize the {\it regret}, namely the number of mistakes compared to the best expert in $\mathcal{H}$:
	\[ \sum_{t=1}^T 1[y_t\neq \hat{y}_t] - \min_{h^*\in\mathcal{H}} \sum_{t=1}^T 1[y_t\neq h^*(x_t)].\]
	In this context, a class $\mathcal H$ is said to be online learnable if for every $T$, 
	there is an algorithm that achieves sublinear regret $o(T)$ against any sequence of $T$ examples.
	The {\it Littlestone dimension} is a combinatorial parameter associated to the class $\cH$ 
	which characterizes its online learnability~\cite{Littlestone87online,Bendavid09agnostic}:
	$\mathcal H$~is online learnable if and only if it has a finite Littlestone dimension~$d<\infty$.
    Moreover, the best possible regret $R(T)$ for online learning of $\cH$ satisfies 
	\[\Omega (\sqrt{dT}) \leq R(T) \leq O(\sqrt{dT\log T}).\]
	Furthermore, if it is known that if one of the experts never errs (i.e.\ in the realizable mistake-bound model), 
	then the optimal regret is exactly $d$.

\paragraph{Stability.}

While at a first glance it may seem that online learning and differentially-private learning have little to do with one another, 
	a line of recent works has revealed a tight connection between the two~\cite{Agarwal17dponline,Abernathy17onlilnedp,AlonLMM19,bousquet2019passing,NeelRW19,Joseph2019TheRO,Gonen19privateonline}.

At a high-level, this connection appears to boil down to the notion of stability, which plays a key role in both topics.
	On one hand, the definition of differential privacy is itself a form of stability;
	it requires robustness of the output distribution of an algorithm when its input undergoes small changes.
	On the other hand, stability also arises as a central motif in online learning paradigms
	such as {\it Follow the Perturbed Leader}~\cite{Kalai02geometricalgorithms,kalai05efficient} and {\it Follow the Regularized Leader}~\cite{abernethy08competing, Shalev07ftrl,Hazan16oco}.
	
In their monograph~\cite{DworkR14}, Dwork and Roth identified stability as a common factor of learning and differential privacy:
	{\it ``Differential privacy is enabled by stability and ensures stability\ldots
	we observe a tantalizing moral equivalence between learnability, differential privacy, and stability.''}
 	This insight has found formal manifestations in several works.
	For example, Abernethy et al.\ used DP inspired stability methodology to derive a unified framework
	for proving state of the art bounds in online learning~\cite{Abernathy17onlilnedp}.
	In the opposite direction, Agarwal and Singh showed that certain standard stabilization techniques
	in online learning imply differential privacy~\cite{Agarwal17dponline}. 
\vspace{1.5mm}

Stability plays a key role in this work as well. Our main result, 
	which shows that any class with a finite Littlestone dimension can be privately learned, 
	hinges on the following form of stability: for $\eta > 0$ and $n\in\mathbb{N}$, 
	a learning algorithm $\cA$ is {\it $(n,\eta)$-globally stable}\footnote{The word {\it global} highlights a difference with other forms of algorithmic stability.
	Indeed, previous forms of stability such as DP and {\it uniform hypothesis stability}~\cite{Bousquet02stability} are local in the sense that they require output robustness 
	subject to {\it local} changes in the input. However, the property required by global stability captures stability with respect to resampling the entire input.} with respect to a distribution $\cD$ over examples
	if there exists an hypothesis $h$ whose frequency as an output is at least $\eta$. Namely,
	\[\Pr_{S\sim \cD^n}[\cA(S) = h] \geq \eta.\]
	Our argument follows by showing that every $\cH$ can be learned by a globally-stable algorithm with parameters $\eta = \correction{\exp(\exp(-d))}{\exp(-d)}, n=\correction{\exp(\exp(d))}{\exp(d)}$,
	where $d$ is the Littlestone dimension of $\cH$.
	As a corollary, we get an equivalence between global stability and differential privacy (which can be viewed as a form of local stability).
	That is, the existence of a globally stable learner for~$\cH$ is equivalent 
	to the existence of a differentially-private learner for it (and both are equivalent to having a finite Littlestone dimension). 

\paragraph{Littlestone Classes.}
It is natural to ask which classes have finite Littlestone dimension.
    First, note that every finite class $\cH$ has Littlestone dimension $d \leq \log\lvert \cH\rvert$.
    There are also many natural and interesting infinite classes with finite Littlestone dimension.
    For example, let $X=\mathbb{F}^n$ be an $n$-dimensional vector space over a field $\mathbb{F}$
    and let $\cH\subseteq\{\pm 1\}^X$ consist of all (indicators of) affine subspaces of dimension $\leq d$.
    The Littlestone dimension of $\cH$ is $d$.
    More generally, any class of hypotheses that can be described by
    polynomial \emph{equalities} of constant degree has finite Littlestone dimension(
    Note that if one replaces ``equalities'' with ``inequalities''
    then the Littlestone dimension may become unbounded while the VC dimension remains bounded. This is demonstrated, e.g., by halfspaces which are captured by polynomial inequalities of degree $1$).
    This can be generalized even further to classes that are definable 
    in {\it stable theories}. This (different, still) notion of stability is deep and well-explored in model theory.
    We refer the reader to~\cite{Chase19modelmachine}, Section 5.1 for more examples of stable theories and the Littlestone classes they correspond to.

\paragraph{Organization.}	
The rest of this manuscript is organized as follows. In  Section~\ref{sec:results} we formally state our main results
	and discuss some implications. Section~\ref{sec:overview} overviews some of the main ideas in the proofs.
	Sections~\ref{sec:preliminaries} - \ref{sec:wrapping}	contain complete proofs, and the last section (Section~\ref{sec:conc})
	concludes the paper with some suggestions for future work.

\subsection{Main Results}\label{sec:results}

We next present our main results.
	We begin with the statements concerning the relationship between online learning and differentially-private learning.
	In Section~\ref{sec:stability} we present and discuss the notion of global stability,
	and finally in Section~\ref{sec:boosting} we discuss an implication for private boosting.
	Throughout this section some standard technical terms are used.
	For definitions of these terms we refer the reader to Section~\ref{sec:preliminaries}
\correction{\footnote{An earlier version of this manuscript claimed an upper bound over the sample complexity that is exponential in the Littlestone dimension. The argument was erranous, and the current version contains a correction, which leads to double-exponential dependence in the Littlestone-dimension.}}{}.

\begin{theorem}[Littlestone Classes are Privately Learnable]\label{thm:main}
Let $\cH\subseteq\{\pm 1\}^X$ be a class with Littlestone dimension $d$,
let $\eps,\delta \in (0, 1)$ be privacy parameters, and let $\alpha,\beta \in (0, 1/2)$ be accuracy parameters.
For
\[n = \correction{O\left(\frac{2^{\tilde{O}(2^d)}+\log 1/\beta\delta}{\alpha\epsilon}\right)}{O\left(\frac{16^d \cdot d^2 \cdot (d + \log(1/\beta\delta))}{\alpha\eps}\right)} = O_d\left(\frac{\log(1/\beta\delta)}{\alpha\eps}\right)\]
there exists an $(\eps,\delta)$-DP learning algorithm such that for every realizable distribution $\cD$,
given an input sample $S\sim \cD^n$, the output hypothesis $f=\cA(S)$ satisfies $\loss_{\cD}(f)\leq \alpha$
with probability at least $1-\beta$, where the probability is taken over $S\sim \cD^n$ as well as the internal randomness of~$\cA$.
\end{theorem}

A similar result holds in the agnostic setting:

\begin{corollary}[Agnostic Learner for Littlestone Classes]\label{thm:agnostic}
Let $\cH\subseteq\{\pm 1\}^X$ be a class with Littlestone dimension $d$,
let $\eps$, and $\delta \in (0, 1)$ be privacy parameters, and let $\alpha,\beta \in (0, 1/2)$ be accuracy parameters.
For
\[n = O\left(\correction{\frac{2^{\tilde{O}(2^d)}+\log (1/\beta\delta)}{\alpha\epsilon}}{\frac{16^d \cdot d^2 \cdot (d + \log(1/\beta\delta))}{\alpha\epsilon}} +\frac{\textrm{VC}(\cH)+\log (1/\beta)}{\alpha^2\epsilon} \right)\]
there exists an $(\eps,\delta)$-DP learning algorithm such that for every distribution $\cD$,
given an input sample $S\sim \cD^n$, the output hypothesis $f=\cA(S)$ satisfies \[\loss_{\cD}(f)\leq \min_{h\in \cH} \loss_{\cD}(h)+ \alpha\]
with probability at least $1-\beta$, where the probability is taken over $S\sim \cD^n$ as well as the internal randomness of~$\cA$.
\end{corollary}

\Cref{thm:agnostic} follows from \Cref{thm:main} by Theorem 2.3 in \cite{alon2020closure}
which provides a general mechanism to transform a learner in the realizable setting to a learner in the agnostic setting\footnote{Theorem 2.3 in \cite{alon2020closure} is based  on a previous realizable-to-agnostic transformation from~\cite{beimel2015learning} which applies to {\it proper} learners. Here we require the more general transformation from~\cite{alon2020closure} as the learner implied by~\Cref{thm:main} may be improper.}.
We note that formally the transformation in \cite{alon2020closure} is stated for a constant $\eps=O(1)$. Taking $\eps=O(1)$ is without loss of generality as a standard ``secrecy-of-the-sample'' argument can be used to convert this learner into one which is $(\eps, \delta)$-differentially private by increasing the sample size by a factor of roughly $1/\eps$ and running the algorithm on a random subsample. See~\cite{KasiviswanathanLNRS11, Vadhan17} for further details.

\begin{theorem}[Private PAC Learning $\equiv$ Online Prediction.]\label{thm:equivalence}
The following statements are equivalent for a class $\cH\subseteq \{\pm 1\}^X$:
\begin{enumerate}
\item $\cH$ is online learnable.
\item $\cH$ is approximate differentially-privately PAC learnable.
\end{enumerate}
\end{theorem}

Theorem~\ref{thm:equivalence} is a corollary of Theorem~\ref{thm:main} (which gives $1\to 2$)
	and the result by~\cite{AlonLMM19} (which gives $2\to 1$).
	We comment that a quantitative relation between the learning and regret rates is also implied:
	it is known that the optimal regret bound for~$\cH$ is $\tilde \Theta_d(\sqrt{T})$,
	where the $\tilde \Theta_d$ conceals a constant which depends on the Littlestone dimension of~$\cH$~\cite{Bendavid09agnostic}.
	Similarly, we get that the optimal sample complexity of agnostically privately learning~$\cH$ is~$\Theta_d(\frac{\log({1}/(\beta\delta))}{\alpha^2\eps})$.
	
We remark however that the above equivalence is mostly interesting from a theoretical perspective, 
	and should not be regarded as an efficient transformation between online and private learning.
	Indeed, the Littlestone dimension dependencies concealed by the $\tilde \Theta_d(\cdot)$ in the above bounds on the regret and sample complexities
	may be very different from one another. For example, there are classes
	for which the $\Theta_d(\frac{\log({1}/(\beta\delta))}{\alpha\eps})$  bound hides a $\mathrm{poly}(\log^*(d))$ dependence, 
	and the $\tilde \Theta_d(\sqrt{T})$  bound hides a $\Theta(d)$ dependence.
	One example which attains both of these dependencies is the class of thresholds over a linearly ordered domain of size~$2^d$~\cite{AlonLMM19,kaplan2019privately}. 

\subsubsection{Global Stability}\label{sec:stability}

Our proof of Theorem~\ref{thm:main}, which establishes that every Littlestone class
	can be learned privately, hinges on an intermediate property which we call {\it global stability}:
	
\begin{definition}[Global Stability]
Let $n\in\mathbb{N}$ be a sample size and $\eta > 0$ be a global stability parameter.
An algorithm $\cA$ is $(n,\eta)$-globally-stable with respect to a distribution $\cD$ 
if there exists an hypothesis $h$ such that 
\[\Pr_{S\sim\cD^n}[A(S) = h] \geq \eta.\]
\end{definition}

While global stability is a rather strong property, it holds automatically for learning algorithms using a finite hypothesis class. By an averaging argument, every learner using $n$ samples which produces a hypothesis in a finite hypothesis class $\cH$ is $(n, 1/|\cH|)$-globally-stable. The following proposition generalizes ``Occam's Razor" for finite hypothesis classes to show that global stability is enough to imply similar generalization bounds in the realizable setting.

\begin{proposition}[Global Stability $\implies$ Generalization]\label{prop:gs}
Let $\cH\subseteq\{\pm 1\}^X$ be a class, and assume that $\cA$ is a \underline{consistent} learner for $\cH$ 
(i.e.\ $\loss_S(\cA(S))=0$ for every realizable sample $S$).
Let~$\cD$ be a realizable distribution such that $\cA$ is $(n,\eta)$-globally-stable with respect to $\cD$, 
and let $h$ be a hypothesis such that $\Pr_{S\sim\cD^n}[A(S) = h] \geq \eta$, as guaranteed by the definition of global stability. 
Then,
\[\loss_\cD(h) \leq \frac{\ln(1/\eta)}{n}.\]
\end{proposition}
\begin{proof}
Let $\alpha$ denote the loss of $h$, i.e.\ $\loss_\cD(h) = \alpha$,
and let $E_1$ denote the event that $h$ is consistent with the input sample $S$.
Thus, $\Pr[E_1] = (1-\alpha)^n$. 
Let $E_2$ denote the event that~$\cA(S)=h$. By assumption, $\Pr[E_2]\geq \eta$.
Now, since $\cA$ is consistent we get that $E_2\subseteq E_1$,
and hence that $\eta\leq(1-\alpha)^n$.
This finishes the proof (using the fact that $1-\alpha \leq e^{-\alpha}$ and taking the logarithm of both sides).
\end{proof}



Another way to view global stability is in the context of \emph{pseudo-deterministic algorithms} \cite{Gat11pseudo}. A pseudo-deterministic algorithm is a randomized algorithm which yields some fixed output with high probability. Thinking of a realizable distribution $\cD$ as an instance on which PAC learning algorithm has oracle access, a globally-stable learner is one which is ``weakly'' pseudo-deterministic in that it produces some fixed output with probability bounded away from zero. A different model of pseudo-deterministic learning, in the context of learning from membership queries, was defined and studied by Oliveira and Santhanam~\cite{OliveiraS18}.

We prove Theorem~\ref{thm:main} by constructing, for a given Littlestone class $\cH$,
an algorithm $\cA$ which is globally stable with respect to \underline{every} realizable distribution.


%

\subsubsection{Boosting for Approximate Differential Privacy}\label{sec:boosting}

Our characterization of private learnability in terms of the Littlestone dimension has new consequences for boosting the privacy and accuracy guarantees of differentially-private learners. Specifically, it shows that the existence of a learner with weak (but non-trivial) privacy and accuracy guarantees implies the existence of a learner with any desired privacy and accuracy parameters --- in particular, one with $\delta(n) = \exp(-\Omega(n))$.

\begin{theorem} \label{thm:adp-boost}
	There exists a constant $c > 0$ for which the following holds. Suppose that for some sample size $n_0$ there is an $(\eps_0, \delta_0)$-differentially private learner $\cal W$ for a class $\cH$ satisfying the guarantee
	\[\Pr_{S\sim \cD^{n_0}}[\loss_{\cD}({\cal W}(S)) > \alpha_0 ] < \beta_0\]
	for $\eps_0 = 0.1, \alpha_0 = \beta_0 = 1/16$, and $\delta_0 \le c/n_0^2\log n_0$.
	
	Then, there exists a constant $C_\cH$ such that for every $\alpha, \beta, \eps, \delta \in (0, 1)$ there exists an $(\eps, \delta)$-differentially private learner for $\cH$ with 
	\[\Pr_{S\sim \cD^{n}}[\loss_{\cD}({\cA}(S)) > \alpha] < \beta\]
	whenever $n \ge C_\cH \cdot \log(1/\beta\delta)/\alpha\eps$.
\end{theorem}

Given a weak learner $\cal W$ as in the statement of Theorem~\ref{thm:adp-boost}, the results of~\cite{AlonLMM19} imply that $\Ldim(\cH)$ is finite. Hence Theorem~\ref{thm:main} allows us to construct a learner for $\cH$ with arbitrarily small privacy and accuracy, yielding Theorem~\ref{thm:adp-boost}. The constant $C_{\cH}$ in the last line of the theorem statement suppresses a factor depending on $\Ldim(\cH)$.

Prior to our work, it was open whether arbitrary learning algorithms satisfying approximate differential privacy could be boosted in this strong a manner. We remark, however, that in the case of \emph{pure} differential privacy, such boosting can be done algorithmically and efficiently. Specifically, given an $(\eps_0, 0)$-differentially private weak learner as in the statement of Theorem~\ref{thm:adp-boost}, one can first apply random sampling to improve the privacy guarantee to $(p\eps_0, 0)$-differential privacy at the expense of increasing its sample complexity to roughly $n_0 /p$ for any $p \in (0, 1)$. The Boosting-for-People construction of Dwork, Rothblum, and Vadhan~\cite{DworkRV10} (see also~\cite{BunCS20}) then produces a strong learner by making roughly $T \approx \log(1/\beta)/\alpha^2$ calls to the weak learner. By composition of differential privacy, this gives an $(\eps, 0)$-differentially private strong learner with sample complexity roughly $n_0 \cdot \log(1/\beta)/\alpha^2\eps$.

What goes wrong if we try to apply this argument using an $(\eps_0, \delta_0)$-differentially private weak learner? Random sampling still gives a $(p\eps_0, p\delta_0)$-differentially private weak learner with sample complexity $n_0 / p$. However, this is not sufficient to improve the $\delta$ parameter of the learner \emph{as a function of the number of samples $n$}. Thus the strong learner one obtains using Boosting-for-People still at best guarantees $\delta(n) = \tilde{O}(1/n^2)$. Meanwhile, Theorem~\ref{thm:adp-boost} shows that the existence of a $(0.1, \tilde{O}(1/n^2))$-differentially private learner for a given class implies the existence of a $(0.1, \exp(-\Omega(n))$-differentially private learner for that class.

We leave it as an interesting open question to determine whether this kind of boosting for approximate differential privacy can be done algorithmically.

\section{Proof Overview}\label{sec:overview}
We next give an overview of the main arguments used in the proof of Theorem~\ref{thm:main}.
The proof consist of two parts:
(i) we first show that every class with a finite Littlestone dimension
can be learned by a globally-stable algorithm, and
(ii) we then show how to generically obtain a differentially-private learner from any globally-stable learner.

\subsection{Step 1: Finite Littlestone Dimension $\implies$ Globally-Stable Learning}

Let $\cH$ be a concept class with Littlestone dimension $d$. Our goal is to design a globally-stable learning algorithm for $\cH$ with stability parameter $\eta = \correction{2^{-2^{O(d)}}}{\exp(-d)}$ and sample complexity $n=\correction{2^{2^{O(d)}}}{\exp(d)}$.
	We will sketch here a weaker variant of our construction which uses the same ideas but is simpler to describe.

The property of $\cH$ that we will use is that it can be online learned in the realizable setting with at most $d$ mistakes 
	(see Section~\ref{sec:online} for a brief overview of this setting). 
	Let $\cD$ denote a realizable distribution with respect to which we wish to learn in a globally-stable manner.
	That is, $\cD$ is a distribution over examples $(x,c(x))$ where $c\in\cH$ is an unknown target concept.
	Let $\mathcal{A}$ be a learning algorithm that makes at most $d$ mistakes while learning an unknown concept from $\cH$ in the online model.
	Consider applying $\mathcal{A}$ on a sequence $S=((x_1,c(x_1))\ldots (x_n,c(x_n)))\sim\cD^n$,
	and denote by $M$ the random variable counting the number of mistakes~$\mathcal{A}$ makes in this process.
	The mistake-bound guarantee on $\cA$ guarantees that~$M\leq d$ always. 
	Consequently, there is $0\leq i \leq d$ such that 
	\[\Pr[M=i] \geq \frac{1}{d+1}.\]
	Note that we can identify, with high probability, an $i$ such that $\Pr[M=i] \geq 1/2d$ by running~$\cA$ on $O(d)$ samples from $\cD^n$.
	We next describe how to handle each of the $d+1$ possibilities for $i$. 
	Let us first assume that $i=d$, namely that
\[\Pr[M=d] \geq \frac{1}{2d}.\]
	We claim that in this case we are done: 
	indeed, after making $d$ mistakes it must be the case that $\cA$ has completely identified the target concept $c$
	(or else $\cA$ could be presented with another example which forces it to make $d+1$ mistakes).
	Thus, in this case it holds with probability at least~$1/2d$ that $\cA(S)=c$ and we are done.
	Let us next assume that $i=d-1$, namely that
\[\Pr[M=d-1] \geq \frac{1}{2d}.\] 
	The issue with applying the previous argument here 
	is that before making the $d$'th mistake, $\cA$ can output many different hypotheses (depending on the input sample $S$).
	We use the following idea: draw two samples $S_1,S_2 \sim \cD^n$ independently, 
	and set $f_1 = \cA(S_1)$ and $f_2=\cA(S_2)$.
	Condition on the event that the number of mistakes made by $\cA$ on each of $S_1,S_2$ is exactly~$d-1$
	(by assumption, this event occurs with probability at least $(1/2d)^2$)
	and consider the following two possibilities: 
	\begin{itemize}
	\item[(i)] $\Pr[f_1=f_2]\geq\frac{1}{4}$, 
	\item[(ii)] $\Pr[f_1=f_2] < \frac{1}{4}$.
	\end{itemize}	
	If (i) holds then using a simple calculation one can show that there is $h$ such that $\Pr[A(S) = h] \geq \frac{1}{(2d)^2}\cdot \frac{1}{4}$ and we are done.
	If (ii) holds then we apply the following {\it ``random contest''} between $S_1,S_2$:
	\begin{enumerate}
	\item Pick $x$ such that $f_1(x)\neq f_2(x)$ and draw $y\sim\{\pm 1\}$ uniformly at random.
	\item If $f_1(x)\neq y$ then the output is $\cA(S_1 \circ (x,y))$, 
	where $S_1\circ (x,y)$ denotes the sample obtained by appending $(x,y)$ to the end of $S$.
	In this case we say that $S_1$ ``won the contest''.
	\item Else, $f_2(x)\neq y$ then the output is  $\cA(S_2 \circ (x,y))$.
	In this case we that $S_2$ ``won the contest''.
	\end{enumerate}
	Note that adding the auxiliary example $(x,y)$ forces $\cA$ to make exactly $d$ mistakes on $S_i\circ (x,y)$.
	Now, if $y\sim\{\pm 1\}$ satisfies $y = c(x) $ then by the mistake-bound argument it holds that $\cA(S_i\circ (x,y))=c$.
	Therefore, since $\Pr_{y\sim\{\pm 1\}}[c(x)=y] = 1/2$, it follows that 
	\[\Pr_{S_1,S_2, y}[\cA(S_i\circ (x,y))=c] \geq \frac{1}{(2d)^2}\cdot \frac{3}{4}\cdot\frac{1}{2} =\Omega(1/d^2),\] 
	and we are done.

Similar reasoning can be used by induction to handle the remaining cases (the next one would be that $\Pr[M=d-2] \geq \frac{1}{2d}$, and so on). \correction{As the number of mistakes reduces, we need to guess more labels, to enforce mistakes on the algorithm. As we guess more labels the success rate reduces, nevertheless we never need to make more then $2^d$ such guesses.  (Note that the random contests performed by the algorithm can naturally be presented using the internal nodes of a binary tree of depth $\leq d$ ). }{}
	The proof we present in Section~\ref{sec:LSstable} is based on a similar idea of performing random contests, 
	although the construction becomes more complex to handle other issues, such as generalization, which were not addressed here. 
	For more details we refer the reader to the complete argument in Section~\ref{sec:LSstable}.

\subsection{Step 2: Globally-Stable Learning $\implies$ Differentially-Private Learning}

Given a globally-stable learner $\cA$ for a concept class $\cH$, we can obtain a differentially-private learner using standard techniques in the literature on private learning and query release. If $\cA$ is a $(\eta, m)$-globally stable learner with respect to a distribution $\cD$, we obtain a differentially-private learner using roughly $m/\eta$ samples from that distribution as follows. We first run $\cA$ on $k \approx 1/\eta$ independent samples, non-privately producing a list of $k$ hypotheses. We then apply a differentially-private ``Stable Histograms'' algorithm~\cite{KorolovaKMN09, BunNS16} to this list which allows us to privately publish a short list of hypotheses that appear with frequency $\Omega(\eta)$. Global stability of the learner $\cA$ guarantees that with high probability, this list contains \emph{some} hypothesis $h$ with small population loss. We can then apply a generic differentially-private learner (based on the exponential mechanism) on a fresh set of examples to identify such an accurate hypothesis from the short list.


\section{Preliminaries}\label{sec:preliminaries}
\subsection{PAC Learning}
We use standard notation from statistical learning; see, e.g.,~\cite{Shalev14book}.
	Let $X$ be any ``domain'' set and consider the ``label'' set~$Y=\{\pm 1\}$. A {\it hypothesis} is a function $h : X\to Y$, which we alternatively write as an element of $Y^X$.
	An {\it example} is a pair $(x, y) \in X\times Y$. A {\it sample} $S$ is a finite sequence of examples.
\begin{definition}[Population \& Empirical Loss]
Let $\cD$ be a distribution over $X \times \{\pm 1\}$. The population loss of a hypothesis~$h : X \to \{\pm 1\}$ is defined by
\[\loss_{\cD}(h) = \Pr_{(x, y) \sim \cD}[h(x) \ne y].\]
Let $S=\bigl((x_i,y_i)\bigr)_{i=1}^n$ be a sample. The empirical loss of $h$ with respect to $S$ is defined by
\[\loss_{S}(h) = \frac{1}{n}\sum_{i=1}^n1[h(x_i)\neq y_i].\]
\end{definition}
Let $\cH\subseteq Y^X$ be a {\it hypothesis class}.
A sample $S$ is said to be {\it realizable by $\cH$} if there is $h\in H$ such that $\loss_S(h)=0$.
A distribution $\cD$ is said to be {\it realizable by $\cH$} if there is~$h\in H$ such that~$\loss_\cD(h)=0$.
A {\it learning algorithm} $A$ is a (possibly randomized) mapping taking input samples to output hypotheses.
We also use the following notation: for samples $S,T$, let~$S\circ T$ denote the combined sample obtained by appending $T$ to the end of $S$.



\subsection{Online Learning}\label{sec:online}
\paragraph{Littlestone Dimension.}
The Littlestone dimension is a combinatorial parameter that captures mistake and regret bounds in online learning \cite{Littlestone87online,Bendavid09agnostic}.\footnote{It appears that the name ``Littlestone dimension'' was coined in~\cite{Bendavid09agnostic}.}
	Its definition uses the notion of {\it mistake trees}. A mistake tree is a binary decision tree whose internal nodes are labeled by elements of $X$.
	Any root-to-leaf path in a mistake tree can be described as a sequence of examples 
	$(x_1,y_1),...,(x_d,y_d)$, where $x_i$ is the label of the $i$'th 
	internal node in the path, and $y_i=+1$ if the $(i+1)$'th node  
	in the path is the right child of the $i$'th node and $y_i = -1$ otherwise.
	We say that a mistake tree $T$ is {\it shattered }by $\cH$ if for any root-to-leaf path
	$(x_1,y_1),...,(x_d,y_d)$ in $T$ there is an $h\in \cH$ such that $h(x_i)=y_i$ for all $i\leq d$ (see Figure~\ref{fig:shatteredtree}).
	The Littlestone dimension of $\cH$, denoted $\Ldim(\cH)$,  is the depth of largest
	complete tree that is shattered by~$\cH$. We say that $\cH$ is a Littlestone class if it has finite Littlestone dimension.

\begin{figure}
\centering
\includegraphics[scale=0.3]{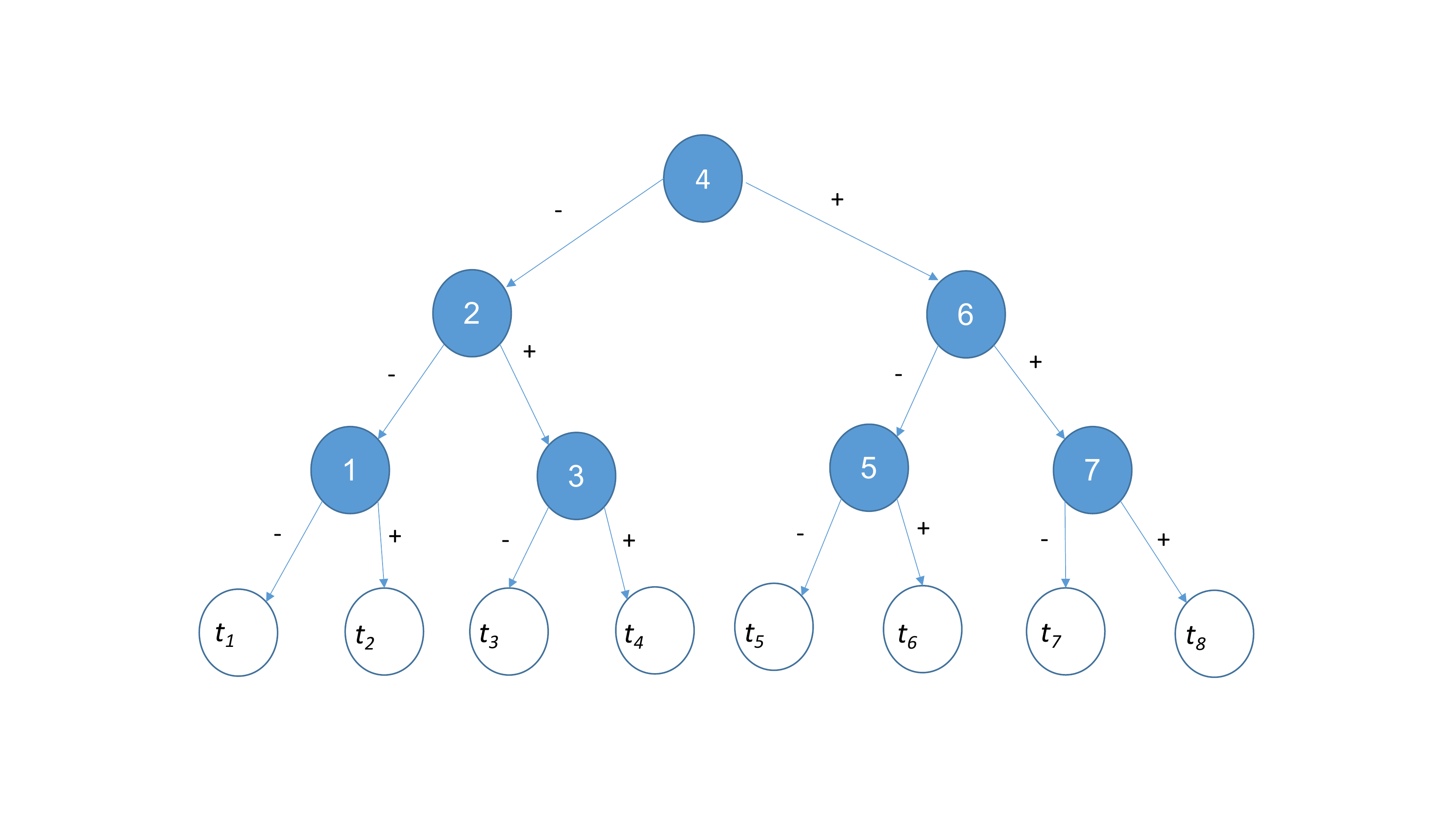}
\caption{\small{A tree shattered by the class  $\cH\subseteq\{\pm 1\}^8$ that contains the threshold functions $t_i$, 
where $t_i(j)=+1$ if and only if $i\leq j$.}}\label{fig:shatteredtree}
\end{figure}

\paragraph{Mistake Bound and the Standard Optimal Algorithm ($\soa$).}
The simplest setting in which learnability is captured by the Littlestone dimension is called the {\it mistake-bound model}~\cite{Littlestone87online}.
	Let $\cH\subseteq \{\pm 1\}^X$ be a fixed hypothesis class known to the learner. 
	The learning process takes place in a sequence of trials, 
	where the order of events in each trial $t$ is as follows: 
	\begin{itemize}
	\item[(i)] the learner receives an instance~$x_t\in X$,
	\item[(ii)] the learner responses with a prediction $\hat y_t\in \{\pm 1\}$, and 
	\item[(iii)] the learner is told whether or not the response was correct.
	\end{itemize}
	We assume that the examples given to the learner are realizable in the following sense: For the entire sequence of trials,
	there is a hypothesis $h\in \cH$ such that $y_t = h(x_t)$ for every instance $x_t$ and correct response $y_t$. An algorithm in this model learns $\cH$ with mistake bound $M$ if for every realizable sequence of examples presented to the learner, it makes a total of at most $M$ incorrect predictions.
	
Littlestone showed that the minimum mistake bound achievable by any online learner is exactly $\Ldim(\cH)$~\cite{Littlestone87online}.
	Furthermore, he described an explicit algorithm, called the {\it Standard Optimal Algorithm} ($\soa$),
	which achieves this optimal mistake bound.
	
\begin{tcolorbox}
	\begin{center}
		{\bf Standard Optimal Algorithm ($\soa$)}\\
	\end{center}
	\begin{enumerate}
		\item Initialize $\cH_1 = \cH$.
		\item For trials $t = 1, 2, \dots$:
		\begin{itemize}
			\item[(i)] For each $b \in \{\pm 1\}$ and $x \in X$, let $\cH_t^b(x) = \{h \in \cH_t : h(x) = b\}$. Define $h : X \to \{\pm 1\}$ by $h_t(x) = \argmax_b \Ldim(\cH_t^{b}(x))$.
			\item[(ii)] Receive instance $x_t$.
			\item[(iii)] Predict $\hat{y}_t = h_t(x_t)$.
			\item[(iv)] Receive correct response $y_t$.
			\item[(v)] Update $\cH_{t+1} = \cH_t^{y_t}(x_t)$.
		\end{itemize}
	\end{enumerate}
\end{tcolorbox}

\paragraph{Extending the $\soa$ to non-realizable sequences.}
Our globally-stable learner for Littlestone classes will make use of an optimal online learner in the mistake bound model.
	For concreteness, we pick the $\soa$ (any other optimal algorithm will also work).
         It will be convenient to extend the $\soa$ to sequences which are not necessarily realizable by a hypothesis in $\cH$.
	We will use the following simple extension of the $\soa$ to non-realizable samples:
\begin{definition}[Extending the $\soa$ to non-realizable sequences]\label{def:soaext}
Consider a run of the $\soa$ on examples $(x_1,y_1),\ldots, (x_m,y_m)$, 
and let $h_t$ denote the predictor used by the $\soa$ after seeing the first $t$ examples 
(i.e., $h_t$ is the rule used by the $\soa$ to predict in the $(t+1)$'st trial).
Then, after observing both $x_{t+1},y_{t+1}$ do the following:
\begin{itemize}
\item If the sequence $(x_1,y_1),\ldots, (x_{t+1},y_{t+1})$ is realizable by some $h\in\cH$ then apply the usual update rule of the $\soa$ to obtain $h_{t+1}$.
\item Else, set $h_{t+1}$ as follows: $h_{t+1}(x_{t+1}) = y_{t+1}$, and $h_{t+1}(x)=h_t(x)$ for every $x\neq x_{t+1}$.
\end{itemize}
\end{definition}
Thus, upon observing a non-realizable sequence, this update rule locally updates the maintained predictor $h_t$ to agree with the last example.

\subsection{Differential Privacy}
We use standard definitions and notation from the differential privacy literature.
For more background see, e.g., the surveys~\cite{DworkR14,Vadhan17}. 
For $a, b, \eps, \delta \in [0, 1]$ let $a\approx_{\eps,\delta} b$ denote the statement
\[a\leq e^{\eps}b + \delta ~\text{  and  }~   b\leq e^\eps a + \delta.\]
We say that two probability distributions $p,q$ are {\it $(\eps,\delta)$-indistinguishable} if 
$p(E) \approx_{\eps,\delta} q(E)$ for every event~$E$.
\begin{definition}[Private Learning Algorithm]\label{def:private}
A randomized algorithm 
\[A: (X\times \{\pm 1\})^m \to \{\pm 1\}^X\] 
is $(\eps,\delta)$-differentially-private 
if for every two samples $S,S'\in (X\times \{\pm 1\})^n$ that disagree on a single example,  
the output distributions $A(S)$ and $A(S')$ are $(\eps,\delta)$-indistinguishable.
\end{definition}
We emphasize that $(\eps, \delta)$-indistinguishability must hold for every such pair of samples, even if they are not generated according to a (realizable) distribution.

The parameters $\eps,\delta$ are usually treated as follows: $\eps$ is a small constant (say $0.1$),  and $\delta$ is negligible, $\delta = n^{-\omega(1)}$, where $n$ is the input sample size. The case of $\delta=0$ is also referred to as {\it pure differential privacy}.
Thus, a class $\cH$ is privately learnable if it is PAC learnable by an algorithm $A$
that is $(\eps(n),\delta(n))$-differentially private with $\eps(n) \leq 0.1$, and~$\delta(n) \leq n^{-\omega(1)} $.


\section{Globally-Stable Learning of Littlestone Classes}\label{sec:LSstable}

\subsection{Theorem Statement}

The following theorem \correction{states that }{ } any class $\cH$ with a bounded Littlestone dimension
	can be learned by a globally-stable algorithm.

\begin{theorem} \label{thm:littlestone-frequent}
Let $\cH$ be a hypothesis class with Littlestone dimension $d\geq 1$, let $\alpha>0$, and set
\[m = \correction{2^{2^{d+2}+1}4^{d+1}}{(8^{d+1}+1)}\cdot\Bigl\lceil\frac{\correction{2^{d+2}}{d}}{\alpha}\Bigr\rceil.\]
Then there exists a randomized algorithm $G : (X \times \{\pm 1\})^m \to \{\pm 1\}^X$ with the following properties. 
Let $\cD$ be a realizable distribution and let $S\sim \cD^m$ be an input sample.
Then there exists a hypothesis $f$ such 
\[\Pr[G(S) = f] \geq \frac{1}{(d+1)2^{\correction{2^d}{d}+1}} \text{ and } \loss_{\cD}(f) \leq \alpha.\]
\end{theorem}

\subsection{The distributions $\cD_k$}

The Algorithm $G$ is obtained by running the $\soa$ on a sample drawn from a carefully tailored distribution.
	This distribution belongs to a family of distributions which we define next. Each of these distributions can be sampled from using black-box access to i.i.d.\ samples from $\cD$.
	Recall that for a pair of samples $S,T$, we denote by $S\circ T$ the sample obtained by appending~$T$ to the end of $S$.
	Define a sequence of distributions $\cD_k$ for $k\geq 0$ as follows:

\begin{tcolorbox}
\begin{center}
{\bf Distributions $\cD_k$}\\
\end{center}
\noindent
Let $n$ denote an ``auxiliary sample'' size (to be fixed later) and let $\cD$ denote the target realizable distribution over examples.
The distributions $\cD_k = \cD_k(\cD,n)$ are defined by induction on $k$ as follows:
\begin{enumerate}
\item $\cD_0$: output the empty sample $\emptyset$ with probability 1.
\item Let $ k\ge 1 $. If there exists a $f$ such that 
\[\Pr_{S \sim \cD_{k-1}, T\sim\cD^n}[\soa(S\circ T) = f] \geq \correction{2^{-2^{d+2}}}{2^{-d}},\]
or if $\cD_{k-1}$ is undefined then $\cD_k$ is undefined.
\item Else, $\cD_k$ is defined recursively by the following process:
	\begin{itemize}
	\item[(i)] Draw $S_0,S_1\sim \cD_{k-1}$ and $T_0,T_1\sim\cD^n$ independently.
	\item[(ii)] Let $f_0=\soa(S_0\circ T_0)$, $f_1=\soa(S_1\circ T_1)$. 
	\item[(iii)] If $f_0=f_1$ then go back to step (i).
	\item[(iv)] Else, pick $x\in \{x: f_0(x)\neq f_1(x)\}$ and sample $y\sim\{\pm 1\}$ uniformly.
	\item[(v)] If $f_0(x)\neq y$ then output $S_0 \circ T_0\circ \bigl((x,y)\bigr)$ and else output $S_1\circ T_1\circ \bigl((x,y)\bigr)$.
	\end{itemize}
\end{enumerate}
\end{tcolorbox}

Please see Figure~\ref{fig:Dk} for an illustration of sampling $S\sim \cD_k$ for $k=3$.

We next observe some basic facts regarding these distributions.
	First, note that whenever~$\cD_k$ is well-defined, the process in Item 3 terminates with probability~1.

Let $k$ be such that $\cD_k$ is well-defined and consider a sample $S$ drawn from $\cD_k$.  
	The size of $S$ is~$\lvert S\rvert = k\cdot(n + 1)$. 
	Among these $k\cdot(n+1)$ examples there are~$k\cdot n$ examples drawn from~$\cD$ 
	and~$k$ examples which are generated in Item 3(iv).
	We will refer to these~$k$ examples as \underline{\it tournament examples}.
	Note that during the generation of $S\sim \cD_k$ there are examples drawn from $\cD$ which do not actually appear in $S$. 
	In fact, the number of such examples may be unbounded, depending on how many times Items~3(i)-3(iii) were repeated. 
	In Section~\ref{sec:monte} we will define a ``Monte-Carlo'' variant of~$\cD_k$ in which the number of examples drawn from $\cD$
	is always bounded. This Monte-Carlo variant is what we actually use to define our globally-stable learning algorithm, but we introduce the simpler distributions $\cD_k$ to clarify our analysis.

The $k$ tournament examples satisfy the following important properties.
\begin{observation}\label{obs:mistakebound}
Let $k$ be such that $\cD_k$ is well-defined and consider running the $\soa$ on the concatenated sample $S\circ T$, 
where~$S\sim \cD_k$ and $T\sim \cD^n$. Then
\begin{enumerate}
\item Each tournament example forces a mistake on the $\soa$.
Consequently, the number of mistakes made by the $\soa$ when run on $S\circ T$
is at least $k$.
\item $\soa(S\circ T)$ is consistent with $T$.
\end{enumerate}
\end{observation}
The first item follows directly from the definition of $x$ in Item 3(iv) and the definition of $S$ in Item 3(v).
The second item clearly holds when $S\circ T$ is realizable by $\cH$ 
(because the $\soa$ is consistent).
For non-realizable $S\circ T$, Item 2 holds by our extension of the $\soa$ in Definition~\ref{def:soaext}.

\begin{figure}
\centering
\includegraphics[scale=0.3]{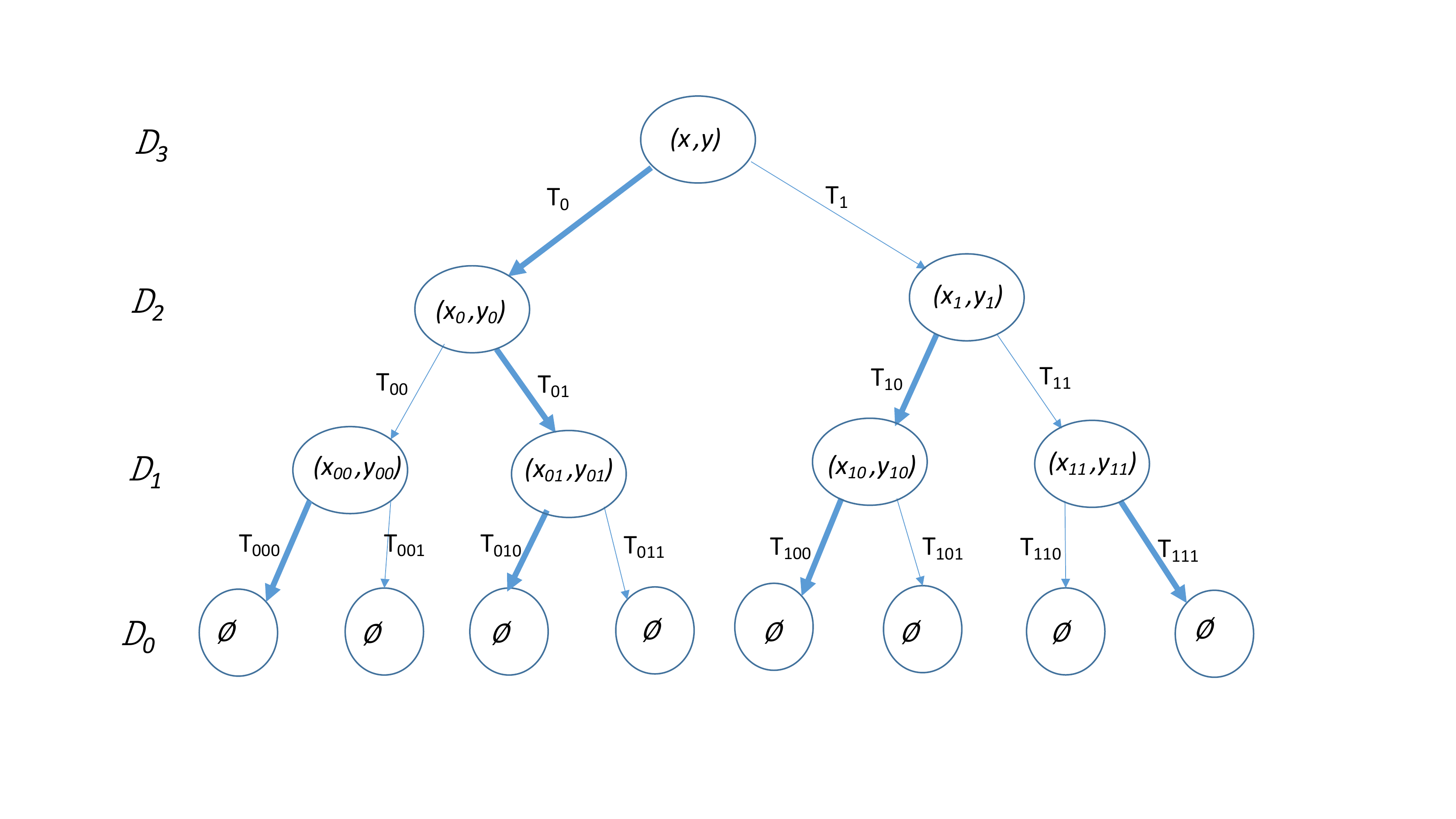}
\caption{\small{An illustration of the process of generating a sample $S\sim \cD_3$. 
The edge labels are the samples $T_b$ drawn in Item~3(i).
The node labels are the tournament examples $(x_b,y_b)$ generated in Item 3(iv). 
The bold edges indicate which of the samples $T_{b0},T_{b1}$ was appended to $S$  in step 3(v) along with the corresponding tournament example.
The sample $S$ generated in this illustration is $T_{010}\circ (x_{01},y_{01})\circ T_{01}\circ (x_{0},y_{0})\circ T_{0}\circ(x,y)$.
}}\label{fig:Dk}
\end{figure}

\subsubsection{The Existence of Frequent Hypotheses}
The following lemma is the main step in establishing global stability.
\begin{lemma}\label{lem:frequentdist}
There exists $k\leq d$ and an hypothesis $f:X\to\{\pm 1\}$ such that 
\[\Pr_{S\sim \cD_k, T\sim\cD^n}[\soa(S\circ T) = f] \geq \correction{2^{-2^{d+2}}}{2^{-d}}.\]
\end{lemma}
\begin{proof}
Suppose for the sake of contradiction that this is not the case.
In particular, this means that $\cD_d$ is well-defined and that for every $f$:

\begin{equation}\label{eq:3}
\Pr_{S\sim \cD_d, T\sim\cD^n}[\soa(S\circ T) = f] < \correction{2^{-2^{d+2}}}{2^{-d}}.
\end{equation}

We show that this cannot be the case when $f=c$ is the target concept (i.e., for $c\in\cH$ which satisfies~$\loss_\cD(c)=0$).
Towards this end, \correction{we first show }{note} that with probability $\correction{2^{-2^{d+2}}}{2^{-d}}$ over $S\sim \cD_d$
we have that all $d$ tournament examples are consistent with~$c$:
\correction{for $k\leq d$ let $\rho_k$ denote the probability that all $k$ tournament examples in $S\sim \cD_k$ are consistent with  $c$.	
	We claim that $\rho_k$ satisfies the recursion $\rho_k\ge \frac{1}{2}(\rho_{k-1}^2-8\cdot2^{-2^{d+2}})$.
	Indeed, consider the event $E_k$ that (i) in each of $S_0,S_1\sim \cD_{k-1}$, all $k-1$ 
	tournament examples are consistent with $c$, and (ii) that $f_0\ne f_1$. 
	Since $f_0=f_1$ occurs with probability at most $2^{-2^{d+2}}<8\cdot2^{-2^{d+2}}$, 
	it follows that $\Pr[E_k]\geq \rho_{k-1}^2- 8\cdot 2^{-2^{d+2}}$. 
	Further, since $y\in \{\pm 1\}$ is chosen uniformly at random and independently of $S_0$ and $S_1$,
	we have that conditioned on $E_k$, $c(x)=y$ with probability~$1/2$. 
	Taken together we have that $\rho_k\geq \frac{1}{2}\Pr[E_k]\geq \frac{1}{2}\left(\rho_{k-1}^2-8\cdot2^{-2^{d+2}}\right)$. 
	Since $\rho_0=1$ we get the recursive relation 
	\[\rho_{k}\ge\frac{\rho_{k-1}^2-8\cdot 2^{-2^{d+2}}}{2}, ~\textrm{and}~ \rho_0=1.\] 
	Thus, it follows by induction that for $k\le d$, $\rho_k\ge 4\cdot2^{-2^{k+1}}$: 
	the base case is verified readily,
	and the induction step is as follows:
	\begin{align*}
	\rho_{k} &\geq \frac{\rho_{k-1}^2-8\cdot2^{-2^{d+2}}}{2}\\
		     &\geq \frac{(4\cdot2^{-2^{k}})^2-8\cdot 2^{-2^{d+2}}}{2} \tag{by induction}\\
		     &=8\cdot 2^{-2^{k+1}} - 4\cdot 2^{-2^{d+2}}\\
		     &\geq 4\cdot 2^{-2^{k+1}}  \tag{$k\leq d$ and therefore $2^{-2^{d+2}} \leq 2^{-2^{k+1}}$}
	\end{align*}  

}{Indeed, this follows since in each tournament example~$(x_i,y_i)$, 
the label $y_i$ is drawn independently of $x_i$ and of the sample constructed thus far.
So, $y_i=c(x_i)$ with probability $1/2$ independently for each tournament example.}

Therefore, with probability $\correction{2^{-2^{d+2}}}{2^{-d}}$ we have that $S\circ T$ is consistent with $c$
(because all examples in $S\circ T$ which are drawn from $\cD$ are also consistent with $c$).
Now, since each tournament example forces a mistake on the $\soa$ (Observation~\ref{obs:mistakebound}),
and since the $\soa$ does not make more than~$d$ mistakes on realizable samples,  
it follows that if all tournament examples in $S\sim \cD_d$ are consistent with $c$ then $\soa(S)=\soa(S\circ T)=c$. 
Thus, 

\[\Pr_{S\sim \cD_d, T\sim\cD^n}[\soa(S\circ T) = c] \geq \correction{2^{-2^{d+2}}}{2^{-d}},\]
which contradicts Equation~\ref{eq:3} and finishes the proof.
\end{proof}

\subsubsection{Generalization}

The next lemma shows that only hypotheses $f$ that generalize well satisfy the conclusion of Lemma~\ref{lem:frequentdist}
(note the similarity of this proof with the proof of Proposition~\ref{prop:gs}):

\begin{lemma}[Generalization]\label{lem:gen}
Let $k$ be such that $\cD_k$ is well-defined. Then every~$f$ such that
\[\Pr_{S\sim \cD_k, T\sim\cD^n}[\soa(S\circ T) = f] \geq \correction{2^{-2^{d+2}}}{2^{-d}}\] 
satisfies $\loss_{\cD}(f) \le \frac{\correction{2^{d+2}}{d}}{ n}$.
\end{lemma}
 \begin{proof}
Let  $f$ be a hypothesis such that $\Pr_{S\sim \cD_k, T \sim \cD^n}[\soa(S \circ T) = f] \geq \correction{2^{-2^{d+2}}}{2^{-d}}$
and let $\alpha=\loss_{\cD}(h)$.
We will argue that 
\begin{equation}\label{eq:4}
\correction{2^{-2^{d+2}}}{2^{-d}} \leq (1-\alpha)^{n}.
\end{equation}
Define the events $A,B$ as follows.
\begin{enumerate}
\item $A$ is the event that $\soa(S\circ T) = f$. By assumption, $\Pr[A] \geq \correction{2^{-2^{d+2}}}{2^{-d}}$.
\item $B$ is the event that $f$ is consistent with $T$. Since $\lvert T\rvert = n$, we have that $\Pr[B] = (1-\alpha)^{n}$.
\end{enumerate}
Note that~$A \subseteq B$:
Indeed, $\soa(S\circ T)$ is consistent with $T$ by  the second item of Observation~\ref{obs:mistakebound}.  
	Thus, whenever $\soa(S\circ T)=f$, it must be the case that $f$ is consistent with $T$.
	Hence,~$\Pr[A]\leq \Pr[B]$, which implies Inequality~\ref{eq:4} and finishes the proof
	(using the fact that $1-\alpha \leq 2^{-\alpha}$ and taking logarithms on both sides).
\end{proof}

\subsection{The Algorithm $G$}

\subsubsection{A Monte-Carlo Variant of $\cD_k$}\label{sec:monte}

Consider the following first attempt of defining a globally-stable learner $G$:
	(i) draw $i\in\{0\ldots d\}$ uniformly at random, (ii) sample $S\sim\cD_i$,
	and (iii) output $\soa(S\circ T)$, where~$T\sim \cD^n$.
	The idea is that with probability $1/(d+1)$ the sampled $i$ will be equal to a number $k$ satisfying the conditions of Lemma~\ref{lem:frequentdist},
	and so the desired hypothesis $f$ guaranteed by this lemma (which also has low population loss by Lemma~\ref{lem:gen}) will be outputted with probability at least $\correction{2^{-2^d}}{2^{-d}}/(d+1)$.

The issue here is that sampling~$f\sim \cD_i$ may require an unbounded number of samples from the target distribution~$\cD$
	(in fact, $\cD_i$ may even be undefined). 	To circumvent this possibility, 
	we define a Monte-Carlo variant of $\cD_k$ in which the number of examples drawn from~$\cD$ is always bounded.

\begin{tcolorbox}
\begin{center}
{\bf The Distributions $\tilde \cD_k$ (a Monte-Carlo variant of $\cD_k$)}\\
\end{center}
\begin{enumerate}
\item 
Let $n$ be the auxiliary sample size and $N$ be an upper bound on the number of examples drawn from $\cD$.
\item $\tilde \cD_0$: output the empty sample $\emptyset$ with probability 1.
\item For $k> 0$, define $\tilde \cD_k$ recursively by the following process:
	\begin{itemize}
	\item[(*)] {\bf Throughout the process, if more than $N$ examples from $\cD$ are drawn 
	(including examples drawn in the recursive calls), then output ``Fail''.}
	\item[(i)] Draw $S_0,S_1\sim \tilde \cD_{k-1}$ and $T_0,T_1\sim\cD^n$ independently. 
	\item[(ii)] Let $f_0=\soa(S_0\circ T_0)$, $f_1=\soa(S_1\circ T_1)$. 
	\item[(iii)] If $f_0=f_1$ then go back to step (i).
	\item[(iv)] Else, pick $x\in \{x: f_0(x)\neq f_1(x)\}$ and sample $y\sim\{\pm 1\}$ uniformly.
	\item[(v)] If $f_0(x)\neq y$ then output $S_0 \circ T_0\circ \bigl((x,y)\bigr)$ and else output $S_1\circ T_1\circ \bigl((x,y)\bigr)$.
	\end{itemize}
\end{enumerate}
\end{tcolorbox}

Note that $\tilde \cD_k$ is well-defined for every $k$, even for $k$ such that $\cD_k$ is undefined
	(however, for such~$k$'s the probability of outputting ``Fail'' may be large).

It remains to specify the upper bound $N$ on the number of examples drawn from $\cD$ in $\tilde \cD_k$. 
Towards this end, we prove the following bound on the expected number of examples from $\cD$
that are drawn during generating $S\sim\cD_k$:
\begin{lemma}[Expected Sample Complexity of Sampling From $\cD_k$]\label{lem:avgsample}
Let $k$ be such that $\cD_k$ is well-defined, and let $M_k$ denote the number of examples from~$\cD$
that are drawn in the process of generating $S\sim \cD_k$. Then,
\[ \mathbb{E}[M_k] \leq 4^{k+1}\cdot n.\]
\end{lemma}
\begin{proof}
Note that $\mathbb{E}[M_0]=0$ as $\cD_0$ deterministically produces the empty sample. 
We first show that for all $0 < i < k$,
\begin{equation}\label{eq:1}
\mathbb{E}[M_{i+1}] \leq 4\mathbb{E}[M_{i}] + 4n,
\end{equation}
and then conclude the desired inequality by induction.

To see why Inequality~\ref{eq:1} holds, let the random variable $R$ denote the number of times Item~3(i) was executed during the 
generation of~$S\sim \cD_{i+1}$. That is, $R$ is the number of times a pair~$S_0,S_1\sim \cD_i$ and a pair $T_0,T_1\sim \cD^n$ were drawn.
Observe that~$R$ is distributed geometrically with success probability $\theta$, where:
\begin{align*}
\theta &= 1 - \Pr_{S_0,S_1, T_0,T_1}\bigl[\soa(S_0\circ T_0) = \soa(S_1\circ T_1)\bigr] \\
	  &= 1 - \sum_{\correction{h}{f}}\Pr_{S, T}\bigl[\soa(S\circ T) = \correction{h}{f}\bigr]^2\\
	  &\geq 1-\correction{2^{-2^{d+2}}}{2^{-d+2}},
\end{align*}
where the last inequality follows because $i< k$ and hence $\cD_i$ is well-defined,
which implies that~$\Pr_{S, T}\bigl[\soa(S\circ T) = \correction{h}{f}\bigr]\leq \correction{2^{-2^{d+2}}}{2^{-d}}$ for all $h$.

Now, the random variable $M_{i+1}$ can be expressed as follows:
\[M_{i+1} = \sum_{j=1}^\infty M_{i+1}^{(j)},\]
where 
\[
M_{i+1}^{(j)} = 
\begin{cases}
0	&\text{if } R < j,\\
\text{$\#$ of examples drawn from $\cD$ in the $j$'th execution of Item 3(i)} &\text{if } R\geq j.
\end{cases}
\]
Thus, $\mathbb{E}[M_{i+1}] = \sum_{j=1}^\infty\mathbb{E}[M_{i+1}^{(j)}]$.
We claim that
\[
\mathbb{E}[M_{i+1}^{(j)}] = (1-\theta)^{j-1}\cdot (2\mathbb{E}[M_i] + 2n).
\]
Indeed, the probability that $R\geq j$ is $(1-\theta)^{j-1}$ and conditioned on $R\geq j$, 
in the $j$'th execution of Item 3(i) two samples from~$\cD_{i}$ are drawn and two samples from $\cD^n$ are drawn.
Thus
\[\mathbb{E}[M_{i+1}] = \sum_{j=1}^\infty(1-\theta)^{j-1}\cdot (2\mathbb{E}[M_i] + 2n)= \frac{1}{\theta}\cdot (2\mathbb{E}[M_i] + 2n) \leq 4\mathbb{E}[M_i] + 4n,\]
where the last inequality is true because $\theta \geq 1- \correction{2^{-2^{d+2}}}{2^{-d+2}}\geq 1/2$.

This gives Inequality~\ref{eq:1}.
Next, using that~$\mathbb{E}[M_0]=0$, a simple induction gives 
\[\mathbb{E}[M_{i+1}]\leq (4+4^2+\ldots+ 4^{i+1})n \leq 4^{i+2}n,\] 
and the lemma follows by taking $i+1=k$.
\end{proof}

\subsubsection{Completing the Proof of Theorem~\ref{thm:littlestone-frequent}}

\begin{proof}[Proof of Theorem~\ref{thm:littlestone-frequent}]
Our globally-stable learning algorithm $G$ is defined as follows.

\begin{tcolorbox}
\begin{center}
{\bf Algorithm $G$}\\
\end{center}
\begin{enumerate}
\item Consider the distribution $\tilde \cD_k$, where the auxiliary sample size is set to $n=\lceil \frac{\correction{2^{d+2}}{d}}{\alpha}\rceil$
and the sample complexity upper bound is set to $N=\correction{2^{2^{d+2}+1}4^{d+1}}{8^{d+1}}\cdot n$.
\item Draw $k\in \{0,1,\ldots, d\}$ uniformly at random.
\item Output $h=\soa(S\circ T)$, where $T\sim \cD^n$ and $S\sim \tilde \cD_k$.
\end{enumerate}
\end{tcolorbox}
First note that the sample complexity of $G$ is $\lvert S\rvert + \lvert T\rvert \leq N+n = (\correction{2^{2^{d+2}+1}4^{d+1}}{8^{d+1}}+1)\cdot\lceil\frac{\correction{2^{d+2}}{d}}{\alpha}\rceil$, as required.
It remains to show that there exists a hypothesis $f$ such that:
\[\Pr[G(S) = f] \geq \frac{\correction{2^{-2^{d+2}}}{2^{-(d+1)}}}{d+1} \text{ and } \loss_{\cD}(f) \leq \alpha.\]
By Lemma~\ref{lem:frequentdist}, there exists $k^*\leq d$ and $f^*$ such that 
\[\Pr_{S\sim \cD_{k^*}, T\sim\cD^n}[\soa(S\circ T) = f^*] \geq \correction{2^{-2^{d+2}}}{2^{-d}}.\]
\correction{We assume $k^*$ is minimal, in particular, $\cD_k$ is well defined for $k\le k^*$}{}. By Lemma~\ref{lem:gen},
\[\loss_{\cD}(f^*) \leq \frac{\correction{2^{d+2}}{d}}{n} \leq \alpha.\]
We claim that $G$ outputs $f^*$ with probability at least $\correction{2^{-2^{d+2}-1}}{2^{-(d+1)}}$.
To see this, let $M_{k^*}$ denote the number of examples drawn from $\cD$ during the generation of $S\sim \cD_{k^*}$.
Lemma~\ref{lem:avgsample} and an application of Markov's inequality yield
\begin{align*}
\Pr\bigl[M_{k^*} > \correction{2^{2^{d+2}+1}\cdot 4^{d+1}}{8^{d+1}}\cdot n\bigr] &\leq \Pr\bigl[M_{k^*} > \correction{2^{2^{d+2}+1}}{2^{d+1}}\cdot 4^{k^*+1}\cdot n\bigr]  \tag{because $k^*\leq d$}\\
							&\leq \correction{2^{-2^{d+2}-1}}{2^{-(d+1)}}. \tag{by Markov's inequality, since $\mathbb{E}[M_{k^*}]\leq 4^{k^*+1}\cdot n$}
\end{align*}
Therefore,
\begin{align*}
\Pr_{S\sim \tilde \cD_{k^*}, T\sim \cD^n}[\soa(S\circ T) = f^*] &= \Pr_{S\sim \cD_{k^*}, T\sim \cD^n}[\soa(S\circ T) = f^*  \text{ and } M_{k^*} \leq \correction{2^{2^d+2}4^{d+1}}{8^{d+1}}\cdot n] \\
											&\geq \correction{2^{-2^{d+2}}}{2^{-d}} - \correction{2^{-2^{d+2}-1}}{2^{-(d+1)}} \correction{= 2^{-2^{d}-1}}{=2^{-(d+1)}}.
\end{align*}
Thus, since $k=k^*$ with probability $1/(d+1)$, it follows that
$G$ outputs $f^*$ with probability at least $\frac{\correction{2^{-2^{d+2}-1}}{2^{-(d+1)}}}{d+1}$ as required.
\end{proof}

\section{Globally-Stable Learning Implies Private Learning}\label{sec:stableprivate}

In this section we prove that any globally-stable learning algorithm yields a differentially-private learning algorithm with finite sample complexity.

\subsection{Tools from Differential Privacy}

We begin by stating a few standard tools from the differential privacy literature which underlie our construction of a learning algorithm.

Let $X$ be a data domain and let $S \in X^n$. For an element $x \in X$, define $\freq_S(x) = \frac{1}{n} \cdot \#\{i \in [n] : x_i = x\}$, i.e., the fraction of the elements in $S$ which are equal to $x$.

\begin{lemma}[Stable Histograms~\cite{KorolovaKMN09, BunNS16}] \label{lem:histograms}
Let $X$ be any data domain. For 
\[n \ge O\left(\frac{\log(1/\eta\beta\delta)}{\eta \eps}\right)\]
there exists an $(\eps, \delta)$-differentially private algorithm $\Hist$ which, with probability at least $1-\beta$, on input $S = (x_1, \dots, x_n)$ outputs a list $L \subseteq X$ and a sequence of estimates $a \in [0, 1]^{|L|}$ such that
\begin{itemize}
\item Every $x$ with $\freq_S(x) \ge \eta$ appears in $L$ and
\item For every $x \in L$, the estimate $a_x$ satisfies $|a_x - \freq_S(x)| \le \eta$.
\end{itemize}
\end{lemma}

Using the Exponential Mechanism of McSherry and Talwar~\cite{McSherryT07}, Kasiviswanathan et al.~\cite{KasiviswanathanLNRS11} described a generic differentially-private learner based on approximate empirical risk minimization.

\begin{lemma}[Generic Private Learner~\cite{KasiviswanathanLNRS11}] \label{lem:generic}
	Let $H \subseteq \{\pm 1\}^X$ be a collection of hypotheses. For
	\[n = O\left(\frac{\log|H|  +\log(1/\beta)}{\alpha \eps}\right)\]
	 there exists an $\eps$-differentially private algorithm $\Gen : (X \times \{\pm 1\})^n \to H$ such that the following holds. Let $\cD$ be a distribution over $(X \times \{\pm 1\})$ such that there exists $h^* \in H$ with
	\[\loss_{\cD}(h^*) \le \alpha.\]
	Then on input $S \sim \cD^n$, algorithm $\Gen$ outputs, with probability at least $1-\beta$, a hypothesis $\hat{h} \in H$ such that
	\[\loss_\cD(\hat{h}) \le 2\alpha.\]
\end{lemma}

Our formulation of the guarantees of this algorithm differ slightly from those of~\cite{KasiviswanathanLNRS11}, so we give its standard proof for completeness.

\begin{proof}[Proof of Lemma~\ref{lem:generic}]
	The algorithm $\Gen(S)$ samples a hypothesis $h \in H$ with probability proportional to $\exp(-\eps n \loss_S(h) / 2)$. This algorithm can be seen as an instantiation of the Exponential Mechanism~\cite{McSherryT07}; the fact that changing one sample changes the value of $\loss_S(h)$ by at most $1$ implies that $\Gen$ is $\eps$-differentially private.
	
	We now argue that $\Gen$ is an accurate learner. Let $E$ denote the event that the sample $S$ satisfies the following conditions:
	\begin{enumerate}
		\item For every $h \in H$ such that $\loss_{\cD}(h) > 2\alpha$, it also holds that $\loss_{S}(h) > 5\alpha/3$, and
		\item For the hypothesis $h^* \in H$ satisfying $\loss_{\cD}(h^*) \le \alpha$, it also holds that $\loss_{S}(h^*) \le 4\alpha / 3$.
	\end{enumerate}
	We claim that $\Pr[E] \ge 1-\beta/2$ as long as $n \ge O(\log(|H|/\beta) / \alpha)$. To see this, let $h \in H$ be an arbitrary hypothesis with $\loss_D(h) > 2\alpha$. By a multiplicative Chernoff bound\footnote{I.e., for independent random variables $Z_1, \dots, Z_n$ whose sum $Z$ satisfies $\Expectation[Z] = \mu$, we have for every $\delta \in (0, 1)$ that $\Pr[Z \le (1-\delta)\mu] \le \exp(-\delta^2\mu / 2)$ and $\Pr[Z \ge (1 + \delta)\mu] \le \exp(-\delta^2\mu / 3)$.} we have $\loss_S(h) > 7\alpha / 4$ with probability at least $1 - \beta/(4|H|)$ as long as $n \ge O(\log(|H|/\beta) / \alpha)$. Taking a union bound over all $h \in H$ shows that condition 1. holds with probability at least $1 - \beta/4$. Similarly, a multiplicative Chernoff bound ensures that condition 2 holds with probability at least $1 - \beta/4$, so $E$ holds with probability at least $1-\beta/2$.
	
	Now we show that conditioned on $E$, the algorithm $\Gen(S)$ indeed produces a hypothesis $h$ with $\loss_D(\hat{h}) \le 2\alpha$. This follows the standard analysis of the accuracy guarantees of the Exponential Mechanism. Condition 2 of the definition of event $E$ guarantees that $\loss_S(h^*) \le 4\alpha / 3$. This ensures that the normalization factor in the definition of the Exponential Mechanism is at least $\exp(-2\eps \alpha n /3)$. Hence by a union bound,
	\[\Pr[\loss_S(\hat{h}) > 5\alpha/3] \le |H| \cdot \frac{\exp(-5\eps \alpha n / 6)}{\exp(-2\eps \alpha n / 3)} = |H| e^{-\eps \alpha n / 6}.\]
	Taking $n \ge O(\log(|H|/\beta) / \alpha\eps)$ ensures that this probability is at most $\beta / 2$. Given that $\loss(\hat{h}) \le 5\alpha / 3$, Condition 1 of the definition of event $E$ ensures that $\loss_{\cD}(\hat{h}) \le 2\alpha$. Thus, for $n$ sufficiently large as described, we have overall that $\loss_{\cD}(\hat{h}) \le 2\alpha$ with probability at least $1- \beta$.
\end{proof}

\subsection{Construction of a Private Learner}

We now describe how to combine the Stable Histograms algorithm with the Generic Private Learner to convert any globally-stable learning algorithm into a differentially-private one.

\begin{theorem} \label{thm:selection}
	Let $\cH$ be a concept class over data domain $X$. Let $G : (X \times \{\pm 1\})^m \to \{\pm 1\}^X$ be a randomized algorithm such that, for $\cD$ a realizable distribution and $S \sim \cD^m$, there exists a hypothesis $h$ such that $\Pr[G(S) = h] \ge \eta$ and $\loss_{\cD}(h) \le \alpha / 2$. 
	
	Then for some
	\[n = O\left(\frac{m \cdot \log(1/\eta\beta\delta)}{\eta\eps} + \frac{\log(1/\eta\beta)}{\alpha\eps}\right)\]
	there exists an $(\eps, \delta)$-differentially private algorithm $M: (X \times \{\pm 1\})^n \to \{\pm 1\}^X$ which, given $n$ i.i.d. samples from $\mathcal{D}$, produces a hypothesis $\hat{h}$ such that $\loss_{\cD}(\hat{h}) \le \alpha$ with probability at least $1-\beta$.
\end{theorem}

Theorem~\ref{thm:selection} is realized the learning algorithm $M$ described below. Here, the parameter
\[k = O\left(\frac{\log(1/\eta\beta\delta)}{\eta\eps}\right)\]
is chosen so that Lemma~\ref{lem:histograms} guarantees Algorithm $\Hist$ succeeds with the stated accuracy parameters. The parameter 
\[n' = O\left(\frac{\log(1/\eta\beta)}{\alpha\eps}\right)\]
is chosen so that Lemma~\ref{lem:generic} guarantees that $\Gen$ succeeds on a list $L$ of size $|L| \le 2/\eta$ with the given accuracy and confidence parameters.

\begin{tcolorbox}
	\begin{center}
		{\bf Differentially-Private Learner $M$}\\
	\end{center}
	\begin{enumerate}
		\item Let $S_1, \dots, S_k$ each consist of $m$ i.i.d.\ samples from $\cD$. Run $G$ on each batch of samples producing $h_1 = G(S_1), \dots, h_k = G(S_k)$.
		\item Run the Stable Histogram algorithm $\Hist$ on input $H = (h_1, \dots, h_k)$ using privacy parameters $(\eps/2, \delta)$ and accuracy parameters $(\eta/8, \beta/3)$, producing a list $L$ of frequent hypotheses.
		\item Let $S'$ consist of $n'$ i.i.d. samples from $\cD$. Run $\Gen(S')$ using the collection of hypotheses $L$ with privacy parameter $\eps / 2$ and accuracy parameters $(\alpha / 2, \beta/3)$ to output a hypothesis $\hat{h}$.
	\end{enumerate}
\end{tcolorbox}

\begin{proof}[Proof of Theorem~\ref{thm:selection}]
We first argue that the algorithm $M$ is differentially private. The outcome $L$ of step 2 is generated in a $(\eps/2, \delta)$-differentially-private manner as it inherits its privacy guarantee from $\Hist$. For every fixed choice of the coin tosses of $G$ during the executions $G(S_1), \dots, G(S_k)$, a change to one entry of some $S_i$ changes at most one outcome $h_i \in H$. Differential privacy for step 2 follows by taking expectations over the coin tosses of all the executions of $G$, and for the algorithm as a whole by simple composition.

We now argue that the algorithm is accurate. Using standard generalization arguments, we have that with probability at least $1-\beta/3$,
\[\left|\freq_H(h) - \Pr_{S \sim \cD^m}[G(S) = h]\right| \le \frac{\eta}{8}\]
for every $h \in \{\pm 1\}^X$ as long as $k \ge O(\log(1/\beta)/\eta)$. Let us condition on this event. Then by the accuracy of the algorithm $\Hist$, with probability at least $1-\beta/2$ it produces a list $L$ containing $h^*$ together with a sequence of estimates that are accurate to within additive error $\eta / 8$. In particular, $h^*$ appears in $L$ with an estimate $a_{h^*} \ge \eta - \eta/8 - \eta/8 \ge 3\eta / 4$.

Now remove from $L$ every item $h$ with estimate $a_h < 3\eta / 4$. Since every estimate is accurate to within $\eta / 8$, this leaves a list with $|L| \le 2/\eta$ that contains $h^*$ with $\loss_{\cD}(h^*) \le \alpha$. Hence, with probability at least $1-\beta/3$, step 3 succeeds in identifying $h^*$ with $\loss_{\cD}(h^*) \le \alpha/2$.

The total sample complexity of the algorithm is $k\cdot m + n'$ which matches the asserted bound.
\end{proof}

\section{Wrapping up (Proof of Theorem~\ref{thm:main})}\label{sec:wrapping}

We now combine Theorem~\ref{thm:littlestone-frequent} (finite Littlestone dimension $\implies$ global stability) with Theorem~\ref{thm:selection} (global stability $\implies$ private learnability) to prove Theorem~\ref{thm:main}.

\begin{proof}[Proof of Theorem~\ref{thm:main}]
	Let $\cH$ be a hypothesis class with Littlestone dimension $d$ and let $\cD$ be any realizable distribution. Then Theorem~\ref{thm:littlestone-frequent} guarantees, for $m = O(\correction{2^{2^{d+2}+1}4^{d+1}}{8^d} \cdot d / \alpha)$, the existence of a randomized algorithm $G : (X \times \{\pm 1\})^m \to \{\pm 1\}^X$ and a hypothesis $f$ such that
	\[\Pr[G(S) = f] \ge \frac{1}{(d+1)2^{\correction{2^{d+2}}{d}+1}} \text{ and } \loss_{\cD}(f) \le \alpha / 2.\]
	Taking $\eta = 1/(d+1)2^{\correction{2^{d+2}}{d}+1}$, Theorem~\ref{thm:selection} gives an $(\eps, \delta)$-differentially private learner with sample complexity
	\[n = O\left(\frac{m \cdot \log(1/\eta\beta\delta)}{\eta\eps} + \frac{\log(1/\eta\beta)}{\alpha\eps}\right) = \correction{O\left(\frac{2^{\tilde{O}(2^d)}+\log 1/\beta\delta}{\alpha\epsilon}\right)}{O\left(\frac{16^d \cdot d^2 \cdot d + \log(1/\beta\delta))}{\alpha\eps}\right)}.\]
\end{proof}

\section{Conclusion}\label{sec:conc}
We conclude this paper with a few suggestions for future work.
\begin{enumerate}
\item {\bf Sharper Quantitative Bounds.} Our upper bound on the differentially-private sample complexity of a class $\cH$ has \correction{a double exponential~}{an exponential} dependence on the Littlestone dimension $\Ldim(\cH)$,
while the lower bound by~\cite{AlonLMM19} depends on $\log^*(\Ldim(\cH))$.
The work by~\cite{kaplan2019privately} shows that for some classes the lower bound is nearly tight (up to a polynomial factor). 
It would be interesting to determine whether the upper bound can be improved in general. 
\correction{
In a followup work to this paper, \cite{Ghazi20efficient} improved the upper bound to $\poly(\Ldim(\cH))$.
However the tower-of-exponents gap between the upper bound and the lower bound remains essentially the same (with 2 levels less).
We thus pose the following question: 
\begin{center}
{\it Can every class $\cH$ be privately learned with sample complexity $\poly(\mathrm{VC}(\cH),\log^\star(\Ldim(\cH)))$?}
\end{center}}{}

\item {\bf Characterizing Private Query Release.} Another fundamental problem in differentially-private data analysis is the query release, or equivalently, data sanitization problem: Given a class $\cH$ and a sensitive dataset $S$, output a synthetic dataset $\hat{S}$ such that $h(S) \approx h(\hat{S})$ for every $h \in \cH$. Does finite Littlestone dimension characterize when this task is possible? Such a statement would follow if one could make our private learner for Littlestone classes \emph{proper}~\cite{bousquet2019passing}.
\item {\bf Oracle-Efficient Learning.} Neel, Roth, and Wu~\cite{NeelRW19} recently began a systematic study of oracle-efficient learning algorithms: Differentially-private algorithms which are computationally efficient when given oracle access to their non-private counterparts. The main open question left by their work is whether \emph{every} privately learnable concept class can be learned in an oracle-efficient manner. Our characterization shows that this is possible if and only if Littlestone classes admit oracle-efficient learners.
\item {\bf General Loss Functions.} It is natural to explore whether the equivalence between online and private learning 
extends beyond binary classification (which corresponds to the 0-1 loss) to regression and other real-valued losses.
\item {\bf Global Stability.} It would be interesting to perform a thorough investigation of global stability and to explore potential connections to other forms of stability in learning theory, including uniform hypothesis stability~\cite{Bousquet02stability},
PAC-Bayes~\cite{McAllester99PACB}, local statistical stability~\cite{Ligett19adaptive}, and others.
\item {\bf Differentially-Private Boosting.} Can the type of private boosting presented in Section~\ref{sec:boosting} be done algorithmically, and ideally, efficiently?
\end{enumerate}

\section*{Acknowledgements}
We would like to thank Amos Beimel and Uri Stemmer for pointing out, and helping to fix, a mistake in the derivation of \cref{thm:agnostic} in a previous version.
We also thank Yuval Dagan for providing useful comments and for insightful conversations.
\bibliographystyle{alpha}
\bibliography{biblio}

\end{document}